\documentclass[10pt,twocolumn,letterpaper]{article}
\usepackage{geometry}
\usepackage{iccv}
\usepackage{times}
\usepackage{epsfig}
\usepackage{graphicx}
\usepackage{amsmath}
\usepackage{amsthm}
\usepackage{amssymb}
\usepackage{bm}
\usepackage{subfigure}
\usepackage[linesnumbered, ruled]{algorithm2e}
\usepackage{booktabs}
\usepackage{multirow}
\usepackage{array}
\usepackage{enumitem}
\usepackage{float}
\usepackage[noblocks]{authblk}
\newtheorem{prop}{Proposition}

\graphicspath{{./}}

\usepackage[breaklinks=true,bookmarks=false]{hyperref}

\iccvfinalcopy 

\def\iccvPaperID{378} 

\newcommand{\Koven}[1]{{\color{black}#1}}
\newcommand{\modify}[1]{{\color{black}#1}}

\newcommand{\ws}[1]{{\color{black}#1}}
\newcommand{\final}[1]{{\color{black}#1}}
\newcommand{\erhao}{\fontsize{21pt}{\baselineskip}\selectfont}

\ificcvfinal\fi
\begin{document}
\newgeometry{top=6cm,bottom=1cm}

\onecolumn{

\noindent \textbf{\erhao{Cross-view Asymmetric Metric Learning for Unsupervised Person Re-identification}}

\vspace{2cm}

\noindent {\LARGE{Hong-Xing Yu, Ancong Wu, Wei-Shi Zheng}}

\Large
\vspace{2cm}

\noindent Code is available at the project page: \\
\ \ \ \ \ \ \ \ \ \ \ \ \url{https://github.com/KovenYu/CAMEL}

\vspace{1cm}

\noindent For reference of this work, please cite:

\vspace{1cm}
\noindent Hong-Xing Yu, Ancong Wu, Wei-Shi Zheng.
``Cross-view Asymmetric Metric Learning for Unsupervised Person Re-identification.''
\emph{Proceedings of the IEEE International Conference on Computer Vision.} 2017.

\vspace{1cm}

\noindent Bib:\\
\noindent
@inproceedings\{yu2017cross,\\
\ \ \   title=\{Cross-view Asymmetric Metric Learning for Unsupervised Person Re-identification\},\\
\ \ \  author=\{Yu, Hong-Xing and Wu, Ancong and Zheng, Wei-Shi\},\\
\ \ \  booktitle=\{Proceedings of the IEEE International Conference on Computer Vision\},\\
\ \ \  year=\{2017\}\\
\}

}

\clearpage

\newpage
\restoregeometry

\twocolumn[{\title{\ws{Cross-view Asymmetric Metric Learning for \\ Unsupervised Person Re-identification}}

\author[ ]{\vspace{-0.5cm}Hong-Xing Yu$^{1,5}$}
\author[ ]{Ancong Wu$^{2}$}
\author[ ]{Wei-Shi Zheng$^{1,3,4}$\thanks{Corresponding author}\vspace{-0.3cm}}

\affil[ ]{$^{1}$School of Data and Computer Science, Sun Yat-sen University, China}
\affil[ ]{$^{2}$School of Electronics and Information Technology, Sun Yat-sen University, China}
\affil[ ]{$^{3}$Key Laboratory of Machine Intelligence and Advanced Computing, Ministry of Education, China}
\affil[ ]{$^{4}$Collaborative Innovation Center of High Performance Computing, NUDT, China}
\affil[ ]{$^{5}$Guangdong Key Laboratory of Big Data Analysis and Processing,  Guangzhou, China}
\affil[ ]{\tt\small xKoven@gmail.com, wuancong@mail2.sysu.edu.cn, wszheng@ieee.org}

\maketitle
}]
\begin{abstract}
While metric learning is important for Person re-identification (RE-ID), a significant problem in visual surveillance for cross-view pedestrian
matching, existing metric models for RE-ID are mostly based on supervised learning that requires quantities of labeled samples in all pairs of camera views for training.
However, this limits their scalabilities to realistic applications, in which a large amount of data
over multiple disjoint camera views is available but not labelled.
To overcome the problem, we propose unsupervised asymmetric metric learning for unsupervised RE-ID. Our model aims to learn an asymmetric metric, i.e., specific projection for each view,
based on asymmetric clustering on cross-view person images.
Our model finds a shared space where view-specific bias is alleviated and thus
better matching performance can be achieved.
Extensive experiments have been conducted on a baseline and five large-scale RE-ID datasets to demonstrate
the effectiveness of the proposed model. Through the comparison, we show that our model works much more suitable for unsupervised RE-ID compared to classical unsupervised metric learning models.
We also compare with existing unsupervised RE-ID methods,
and our model outperforms them with notable margins. Specifically, we report the results on large-scale unlabelled RE-ID dataset, which is important but unfortunately less concerned in literatures.
\end{abstract}
\section{Introduction}\label{Sec1}

\thispagestyle{empty}

Person re-identification (RE-ID) is a challenging problem focusing on pedestrian
matching and ranking across non-overlapping camera views. It remains an open
problem although it has received considerable exploration recently, in consideration of its potential
significance in security applications, especially in the case of video surveillance.
It has not been solved yet principally because of the dramatic intra-class variation
and the high inter-class similarity.
Existing attempts mainly focus on learning to extract robust and discriminative
representations
\cite{2014_ECCV_SCNCD,2014_IVC_KBICOV, 2015_CVPR_LOMO},
and learning matching functions or metrics
\cite{2011_CVPR_PRDC,2012_CVPR_KISSME,2013_CVPR_LADF,2014_ICDSC_KCCA,2015_CVPR_LOMO,2015_ICCV_MLAPG, 2015_ICCV_CSL}
in a supervised manner. Recently, deep learning has been adopted to RE-ID community
\cite{2015_CVPR_Ahmed,2016_CVPR_JSTL,2016_CVPR_Wang, 2016_ECCV_Gated}
and has gained promising results.

However, supervised strategies are intrinsically limited due to the requirement
of manually labeled cross-view training data, which is very expensive \cite{2015_TCSVT_xiaojuan}.
In the context of RE-ID,
the limitation is even pronounced because \emph{(1)} manually labeling may not be reliable
with a huge number of images to be checked across multiple camera views, and more importantly \emph{(2)} the astronomical
cost of time and money is prohibitive to label the overwhelming amount of data across disjoint camera views.
Therefore, in reality supervised methods would be restricted
when applied to a new scenario with a huge number of unlabeled data.

\begin{figure}\label{FigTitle}
\includegraphics[width=1\linewidth]{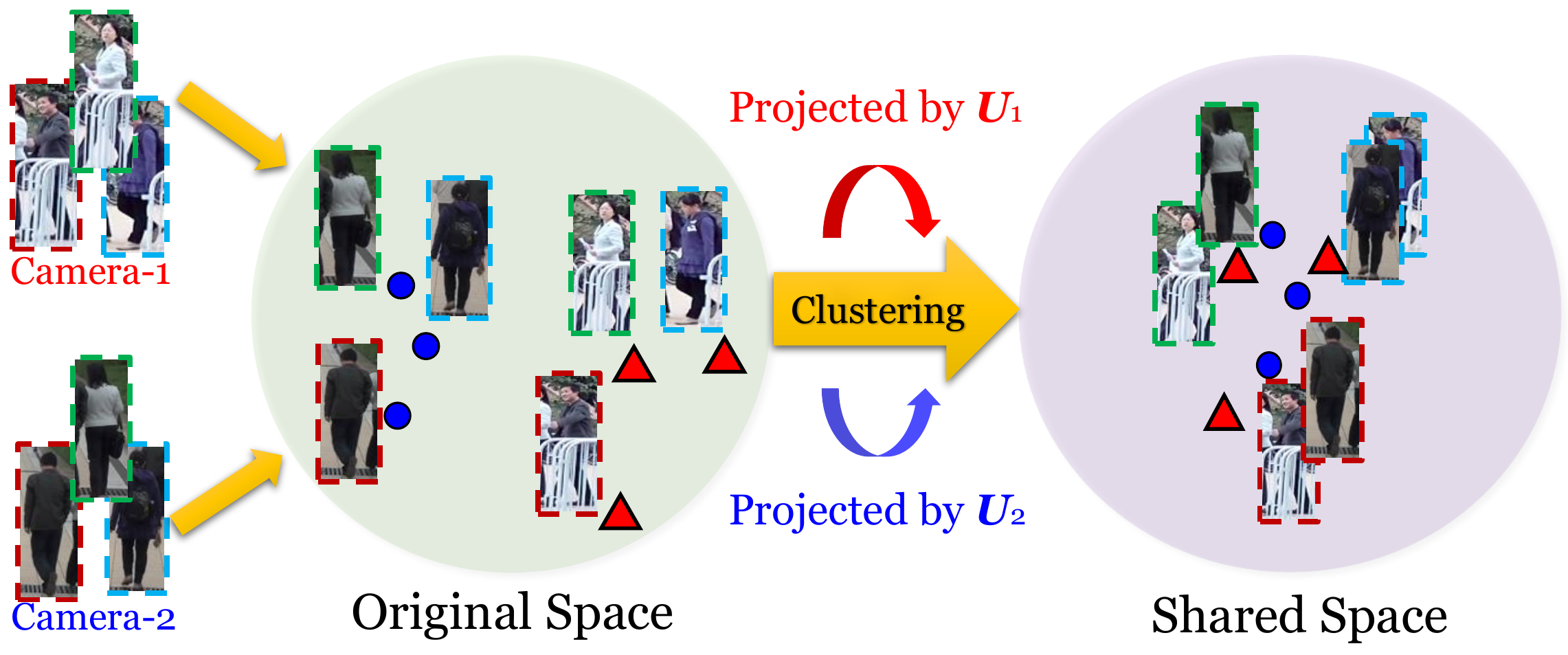}
\caption{Illustration of view-specific interference/bias and our idea.
Images from different cameras suffer from
view-specific interference, such as occlusions in Camera-1,
dull illumination in Camera-2, and the change of viewpoints between them.
These factors introduce bias in the original feature space, and therefore
unsupervised re-identification is extremely challenging. Our model
structures data by clustering and learns view-specific projections
jointly, and thus finds a shared space where view-specific bias is
alleviated and better performance can be achieved. (Best viewed in color)
}
\end{figure}

\ws{To directly make full use of the cheap and valuable unlabeled data,
some existing efforts on exploring unsupervised strategies
\cite{2010_CVPR_SDALF,2013_CVPR_SALIENCE,2014_BMVC_GTS, 2015_BMVC_DIC,2015_PAMI_ISR,2016_CVPR_tDIC, 2016_ICIP_Wang, 2016_ECCV_Kodirov} have been reported,}
but they are still not very satisfactory.
One of the main reasons is that without the help of labeled data,
it is rather difficult to model the dramatic variances
across camera views, such as the variances of illumination and occlusion conditions.
Such variances lead to view-specific interference/bias which can be very disturbing in finding
what is more distinguishable in matching people across views (see Figure \ref{FigTitle}).
In particular, existing unsupervised models treat the samples from different views in the same manner,
and thus the effects of view-specific bias could be overlooked.

In order to better address the problems \ws{caused by camera view changes} in unsupervised RE-ID scenarios, we propose a novel
unsupervised RE-ID model named \emph{Clustering-based Asymmetric\footnote{\final{``Asymmetric'' means specific transformations for each camera view.}} MEtric Learning (CAMEL)}.
The ideas behind are on the two \ws{following} considerations. \ws{First, although}
conditions can vary among camera views, we assume that there should be some shared space
where the data representations are less affected by view-specific bias.
By projecting original data into the shared space, the distance between any pair of
samples $\mathbf{x}_i$ and $\mathbf{x}_j$ is computed as:
\begin{equation}\label{EqSym}
 \small
d(\mathbf{x}_i,\mathbf{x}_j) = \lVert \bm{U}^{\mathrm{T}}\mathbf{x}_i - \bm{U}^{\mathrm{T}}\mathbf{x}_j \rVert_2
= \sqrt{(\mathbf{x}_i-\mathbf{x}_j)^{\mathrm{T}}\bm{M}(\mathbf{x}_i-\mathbf{x}_j)},
\end{equation}
where $\bm{U}$ is the transformation matrix and $\bm{M} = \bm{U}\bm{U}^{\mathrm{T}}$.
\Koven{However, it can be hard for a universal transformation to implicitly model the view-specific feature distortion from different camera views,
especially when we lack label information to guide it.
This motivates us to \emph{explicitly} model the view-specific bias.
Inspired by the supervised asymmetric distance model \cite{2015_TCSVT_ASM},
we propose to embed the asymmetric metric learning to our unsupervised RE-ID modelling,
and thus modify the symmetric form in Eq. (\ref{EqSym}) to an asymmetric one:}
\begin{equation}\label{EqAsym}
\small
d(\mathbf{x}_i^p,\mathbf{x}_j^q) = \lVert \bm{U}^{p\mathrm{T}}\mathbf{x}_i^p - \bm{U}^{q\mathrm{T}}\mathbf{x}_j^q \rVert_2,
\end{equation}
where $p$ and $q$ are indices of camera views.

An asymmetric metric is more acceptable for unsupervised RE-ID scenarios as
it explicitly models the variances among views by treating each view differently.
By such an explicit means, we are able to better alleviate the disturbances of view-specific
bias.

The other consideration is that since we are not clear about how to separate similar persons
in lack of labeled data, it is reasonable to pay more attention to
better separating dissimilar ones.
Such consideration \ws{motivates} us to structure our data by clustering.
Therefore, we develop \emph{asymmetric metric clustering} that clusters cross-view person images.
By clustering together with asymmetric modelling, the data can be better characterized in the shared space,
contributing to better matching performance (see Figure \ref{FigTitle}).

In summary, the proposed CAMEL aims to learn view-specific projection for each camera view
by jointly learning the asymmetric metric and
seeking \ws{optimal} cluster separations.
In this way, the data from different views is projected into
a shared space where view-specific bias is aligned to an extent, and thus better performance
of cross-view matching can be achieved.

\ws{So far in literatures, the unsupervised RE-ID models have only been evaluated on small datasets which contain only hundreds or
a few thousands of images. However, in more realistic scenarios we need evaluations
of unsupervised methods on much larger datasets, say, consisting of hundreds of thousands of samples,
to validate their scalabilities. In our experiments, we have conducted extensive comparison on datasets
with their scales ranging widely.
In particular, we combined two existing RE-ID datasets \cite{2015_ICCV_MARKET,MARS}
to obtain a larger one which contains over 230,000 samples.
Experiments on this dataset (see Sec. \ref{SecFurtherEval}) show empirically that our model is more scalable to problems of larger scales,
which is more realistic and more meaningful for unsupervised RE-ID models, while some existing unsupervised RE-ID models are not scalable due to the expensive cost in either storage or computation.}



\section{Related Work}\label{Sec2}

At present, most existing RE-ID models are in a supervised manner. They are mainly
based on learning distance metrics or subspace
\cite{2011_CVPR_PRDC,2012_CVPR_KISSME,2013_CVPR_LADF,2014_ICDSC_KCCA,2015_CVPR_LOMO,2015_ICCV_MLAPG, 2015_ICCV_CSL},
learning view-invariant and discriminative features
\cite{2014_ECCV_SCNCD,2014_IVC_KBICOV, 2015_CVPR_LOMO},
and deep learning frameworks
\cite{2015_CVPR_Ahmed,2016_CVPR_JSTL,2016_CVPR_Wang, 2016_ECCV_Gated}.
However, all these models rely on substantial labeled training data, which is typically required
to be pair-wise for each pair of camera views. Their performance depends highly on
the quality and quantity of labeled training data.
In contrast, our model does not require any labeled data and thus is free from
prohibitively high cost of manually labeling and the risk of incorrect labeling.

\ws{To directly utilize unlabeled data for RE-ID, several unsupervised RE-ID models \cite{2013_CVPR_SALIENCE,2014_BMVC_GTS,2015_PAMI_ISR,2015_BMVC_DIC,2016_CVPR_tDIC}
have been proposed}.
All these models differ from ours in two aspects.
On the one hand, these models do not explicitly exploit the information on
view-specific bias, i.e., they treat feature transformation/quantization in every distinct camera view in the same manner
when modelling. In contrast, our model tries to learn specific transformation
for each camera view, aiming to find a shared space where view-specific interference
can be alleviated and thus better performance can be achieved.
On the other hand, as for the means to learn a metric or a transformation,
existing unsupervised methods for RE-ID rarely consider clustering while
we introduce an asymmetric metric clustering to characterize data in the learned space. \ws{While the methods proposed in \cite{2015_TCSVT_ASM, 2013_AVSS_RCCA,2015_TCSVT_RCCA} could
also learn view-specific mappings, they are supervised methods and more importantly cannot be generalized to handle unsupervised RE-ID.}


Apart from our model, there have been some clustering-based metric learning models
\cite{2007_CVPR_AML,2015_NC_uNCA}. However, to our best knowledge, there is no such
attempt in RE-ID community before.
This is potentially because clustering is more susceptible to view-specific interference
and thus data points from the same view are more inclined to be clustered together,
instead of those of a specific person across views.
Fortunately, \ws{by formulating asymmetric learning and further limiting the discrepancy between view-specific transforms}, this problem can be
alleviated in our model. Therefore, our model is essentially different from these models
not only in formulation but also
in that our model is able to better deal with cross-view matching problem by treating
each view asymmetrically. We will discuss the differences between our model and the
existing ones in detail in Sec. \ref{SecFairCmp}.

\section{Methodology}

\subsection{Problem Formulation}

Under a conventional RE-ID setting, suppose we have a surveillance camera network that
consists of $V$ camera views, from each of which we have collected
$N_p\;(p = 1,\cdots,V)$ images and thus there are $N = N_1+\cdots+N_V$ images in total as training samples.

Let \modify{ $\bm{X} = [\mathbf{x}_1^1,\cdots,\mathbf{x}_{N_1}^1,\cdots,\mathbf{x}_{1}^V,\cdots,\mathbf{x}_{N_V}^V]\in \mathbb{R}^{M \times N}$}
denote the training set, with each column $\mathbf{x}_i^p$ $(i = 1,\cdots,N_p; p = 1,\cdots,V)$
corresponding to an $M$-dimensional representation of the $i$-th image from the $p$-th
camera view.
Our goal is to learn $V$ mappings i.e., $\bm{U}^1,\cdots,\bm{U}^V$,
where $\bm{U}^p \in \mathbb{R}^{M \times T} (p = 1,\cdots,V)$,
corresponding to each camera view,
and thus we can project the original representation $\mathbf{x}_i^p$
from the original space $\mathbb{R}^M$
into a shared space $\mathbb{R}^T$
in order to alleviate the view-specific interference.

\subsection{Modelling}\label{Sec3}

Now we are looking for some transformations to map our data
into a shared space where we can better separate the
images of one person from those of different persons.
Naturally, this goal can be achieved by narrowing intra-class discrepancy and meanwhile
pulling the centers of all classes away from each other.
In an unsupervised scenario, however, we have no labeled data to tell our model
how it can exactly distinguish one person from another who has a confusingly similar
appearance with him.
Therefore, it is acceptable to relax the original idea:
we focus on gathering similar person images together, and hence separating relatively dissimilar ones.
Such goal can be modelled by minimizing an objective function like that of $k$-means
clustering \cite{KMEANS}:
\begin{equation}\label{Eq0}
 \small
\begin{aligned}
\mathop{\min}_{\bm{U}}\mathcal{F}_{intra}= \sum_{k=1}^K \sum_{i \in {\mathcal{C}_k}} \lVert \bm{U}^{\mathrm{T}}\mathbf{x}_i - \mathbf{c}_k \rVert^2,
\end{aligned}
\end{equation}
where $K$ is the number of clusters,
$\mathbf{c}_k$ denotes the centroid of the $k$-th cluster and
$\mathcal{C}_k = \{ i | \bm{U}^{\mathrm{T}}\mathbf{x}_i \in k$-th cluster$\}$.

However, clustering results may be affected extremely
by view-specific bias when applied in cross-view problems.
In the context of RE-ID, the feature distortion could be view-sensitive due to view-specific interference like
different lighting conditions and occlusions \cite{2015_TCSVT_ASM}.
Such interference
might be disturbing or even dominating in searching the similar person images across views during
clustering procedure. To address this cross-view problem,
we learn specific projection for each view rather than a universal one
to explicitly model the effect of view-specific interference and to alleviate it.
Therefore, the idea can be further formulated
by minimizing an objective function below:
\begin{equation}\label{Eq1}
 \small
\begin{aligned}
\mathop{\min}_{\bm{U}^1,\cdots,\bm{U}^V}\mathcal{F}_{intra}= &\sum_{k=1}^K \sum_{i \in {\mathcal{C}_k}} \lVert \bm{U}^{p\mathrm{T}}\mathbf{x}_i^p - \mathbf{c}_k \rVert^2\\
s.t.\qquad \bm{U}^{p\mathrm{T}}&\bm{\Sigma}^p\bm{U}^p = \bm{I} \quad (p = 1,\cdots,V),
\end{aligned}
\end{equation}
where the notation is similar to Eq. (\ref{Eq0}), with
$p$ denotes the view index,
$\bm{\Sigma}^p = \bm{X}^p\bm{X}^{p\mathrm{T}}/ N_p + \alpha \bm{I}$ and $\bm{I}$ represents the identity matrix
which avoids singularity of the covariance matrix.
The transformation $\bm{U}^p$ that corresponds to each instance $\mathbf{x}_i^p$ is determined
by the camera view which $\mathbf{x}_i^p$ comes from.
The quasi-orthogonal constraints on $\bm{U}^p$ ensure that the model will
not simply give zero matrices. By combining the asymmetric metric learning, we actually realize an asymmetric metric clustering on RE-ID data across camera views.

Intuitively, if we minimize this objective function directly,
$\bm{U}^p$ will largely depend on the data distribution
from the $p$-th view. Now that there is specific bias on each view,
any $\bm{U}^p$ and $\bm{U}^q$ could be arbitrarily different.
This result is very natural,
but large inconsistencies among the learned transformations are
not what we exactly expect,
because the transformations are with respect to person images from different views: they are inherently correlated and homogeneous.
More critically, largely different projection basis pairs would fail to
capture the discriminative nature of cross-view images, producing an even worse
matching result.

Hence, to strike a balance between the ability to capture discriminative nature and
the capability to alleviate view-specific bias, we embed a cross-view consistency regularization term
into our objective function. And then, in consideration of better tractability,
we divide the intra-class term by its scale $N$, so that the regulating parameter
would not be sensitive to the number of training samples.
Thus, our optimization task becomes
\modify{
\begin{equation}\label{Eq2}
  \small
\begin{aligned}
\mathop{\min}_{\bm{U}^1,\cdots,\bm{U}^V} \mathcal{F}_{obj} = \frac{1}{N}&\mathcal{F}_{intra} + \lambda\mathcal{F}_{consistency} \\
= \frac{1}{N}\sum_{k=1}^K &\sum_{i \in {\mathcal{C}_k}} \lVert \bm{U}^{p\mathrm{T}}\mathbf{x}_i^p - \mathbf{c}_k \rVert^2
 +\lambda \sum_{p\neq q} \lVert \bm{U}^p-\bm{U}^q\rVert_F^2 \\
 s.t.\qquad &\bm{U}^{p\mathrm{T}}\bm{\Sigma}^p\bm{U}^p = \bm{I} \quad (p = 1,\cdots,V),
\end{aligned}
\end{equation}}
where $\lambda$ is the cross-view regularizer and $\lVert\cdot\rVert_F$ denotes the Frobenius norm
of a matrix. We call the above model the \emph{Clustering-based Asymmetric MEtric Learning (CAMEL)}.

To illustrate the differences between symmetric and asymmetric metric clustering in structuring data
in the RE-ID problem,
we further show the data distributions in Figure \ref{FigP}.
We can observe from Figure \ref{FigP} that the view-specific
bias is obvious in the original space: triangles in the upper left and circles in the lower right.
In the common space
learned by symmetric metric clustering, the bias is still obvious.
In contrast, in the shared space learned by asymmetric metric clustering,
the bias is alleviated and thus the data is better characterized according to the identities
of the persons, i.e., samples of one person (one color) gather together into a cluster.

\begin{figure}[t]
\hspace{-1ex}
\subfigure[Original]{
\includegraphics[width=0.33\linewidth]{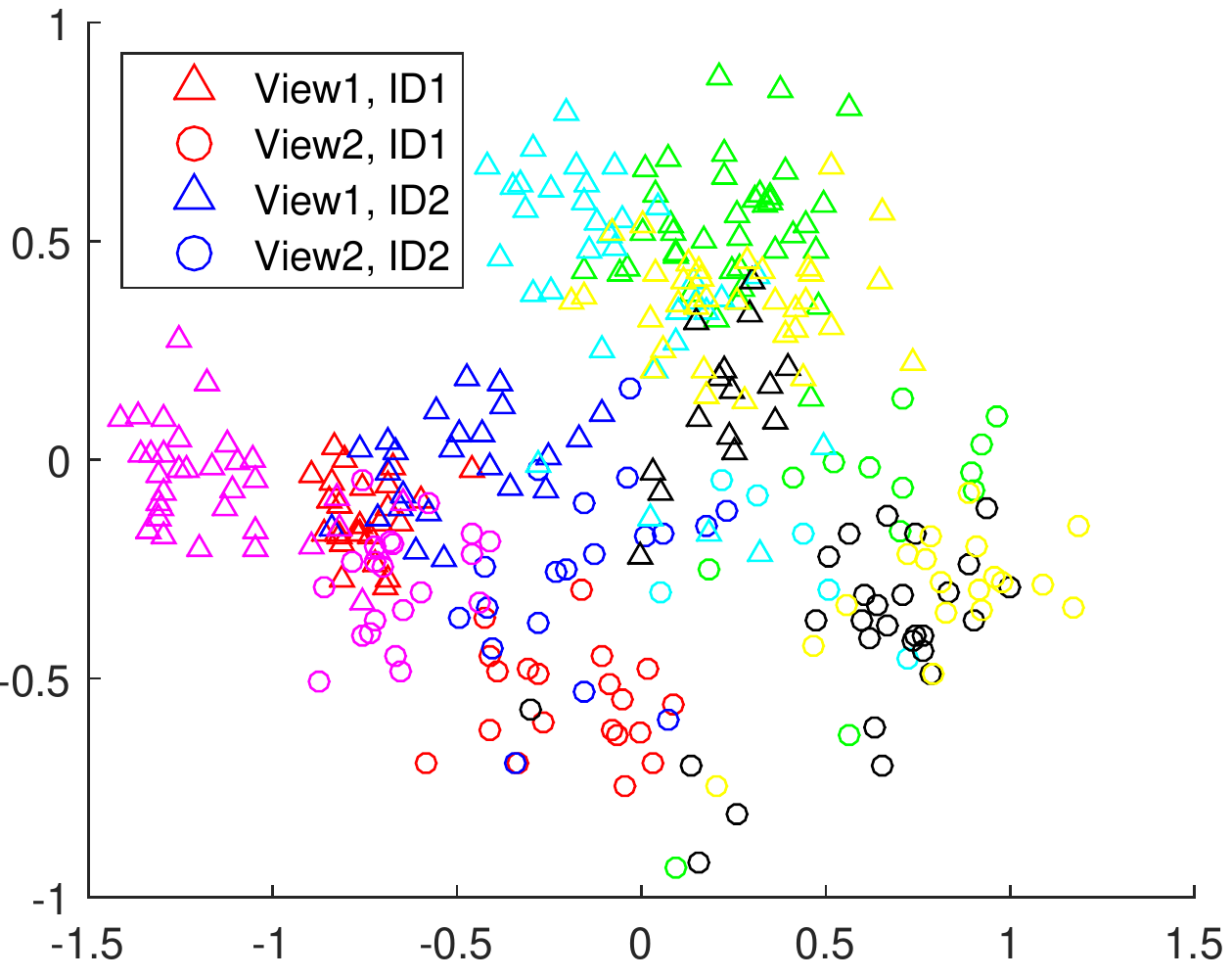}
}
\hspace{-2.5ex}
\subfigure[Symmetric]{
\includegraphics[width=0.33\linewidth]{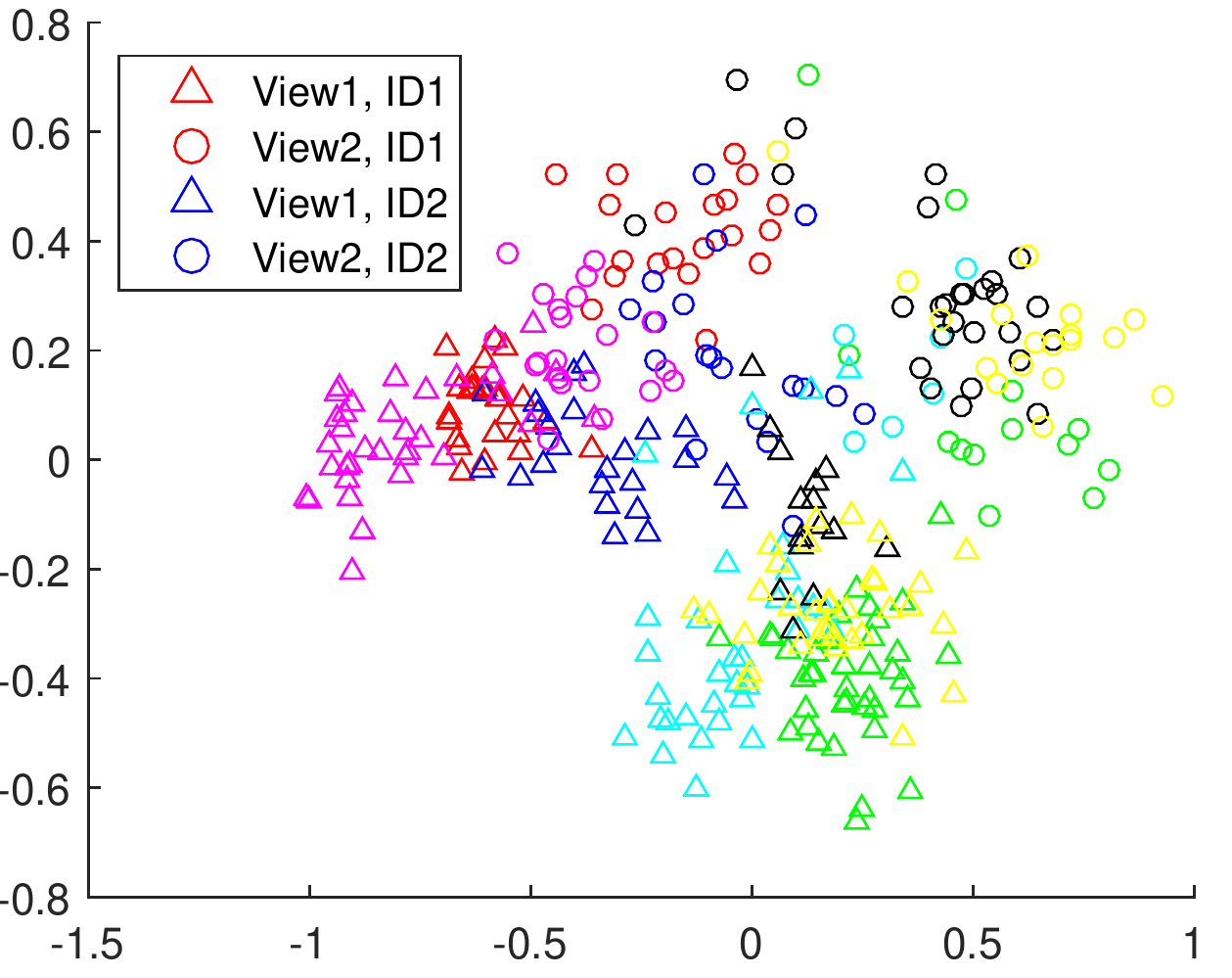}
}
\hspace{-2.5ex}
\subfigure[Asymmetric]{
\includegraphics[width=0.33\linewidth]{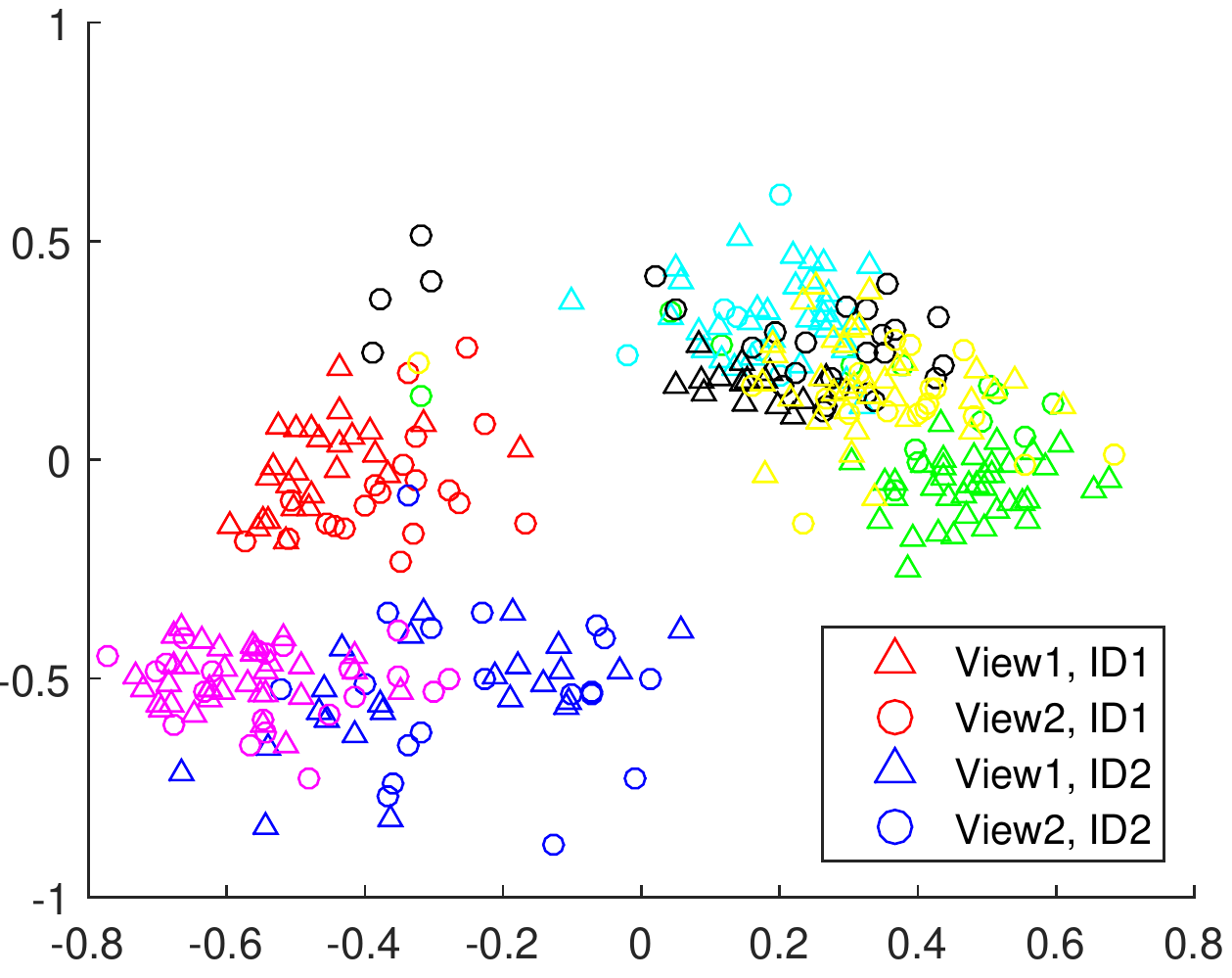}
}
\caption{\label{FigP}Illustration of how symmetric and asymmetric metric clustering structure data
using our method for the unsupervised RE-ID problem. The samples are from the SYSU dataset \cite{2015_TCSVT_ASM}.
We performed PCA for visualization. One shape (triangle or circle) stands for samples from one view, while one color indicates samples of one person.
(a) Original distribution (b) distribution in the common space learned by symmetric metric clustering
(c) distribution in the shared space learned by asymmetric metric clustering. (Best viewed in color)}
\end{figure}

\subsection{Optimization}

For convenience, we denote $\mathbf{y}_i=\bm{U}^{p\mathrm{T}}\mathbf{x}_i^p$. Then we have $\bm{Y} \in \mathbb{R}^{T \times N}$,
where each column $\mathbf{y}_i$
corresponds to the projected new representation of that from $\bm{X}$. For optimization, we rewrite our objective function in a more compact form.
The first term can be rewritten as follow \cite{NMF}:
\begin{equation}\label{Eq3}
 \small
\begin{aligned}
\frac{1}{N}\sum_{k=1}^K \sum_{i \in {\mathcal{C}_k}} \lVert \mathbf{y}_i - \mathbf{c}_k \rVert^2
=\frac{1}{N}[\mathrm{Tr}(\bm{Y}^{\mathrm{T}}\bm{Y})-\mathrm{Tr}(\bm{H}^{\mathrm{T}}\bm{Y}^{\mathrm{T}}\bm{YH})], \\
\end{aligned}
\end{equation}
where
\begin{equation}\label{EqH}
 \small
\bm{H} =
\begin{bmatrix}
\mathbf{h}_1,...,\mathbf{h}_K
\end{bmatrix}
,\quad \mathbf{h}_k^{\mathrm{T}}\mathbf{h}_l =
\begin{cases}
0 & k\neq l \\
1 & k= l
\end{cases}
\end{equation}
\begin{equation}\label{EqColH}
 \small
\mathbf{h}_k =
\begin{bmatrix}
0,\cdots,0,1,\cdots,1,0,\cdots,0,1,\cdots
\end{bmatrix}
^{\mathrm{T}}/\sqrt{n_k}
\end{equation}
is an indicator vector with the $i$-th entry corresponding to the instance $\mathbf{y}_i$,
indicating that $\mathbf{y}_i$ is in the $k$-th cluster if the corresponding entry does not equal zero.
Then we construct
\modify{
\begin{equation}
 \small
\widetilde {\bm{X}} =
\begin{bmatrix}
\mathbf{x}^1_1&\cdots&\mathbf{x}^1_{N_1}& \mathbf{0}& \cdots& \mathbf{0}& \cdots& \mathbf{0} \\
\mathbf{0}&\cdots&\mathbf{0}& \mathbf{x}^2_1&\cdots& \mathbf{x}^2_{N_2}& \cdots& \mathbf{0} \\
\vdots&\vdots&\vdots& \vdots&\vdots& \vdots& \vdots& \vdots \\
\mathbf{0}&\cdots&\mathbf{0}& \mathbf{0}&\cdots& \mathbf{0}& \cdots& \mathbf{x}^V_{N_V}
\end{bmatrix}
\end{equation}}
\begin{equation}
 \small
\widetilde {\bm{U}} =
\begin{bmatrix}
\bm{U}^{1\mathrm{T}}, \cdots, \bm{U}^{V\mathrm{T}}
\end{bmatrix}
^{\mathrm{T}}
,
\end{equation}
so that
\begin{equation}\label{EqY}
 \small
\bm{Y} = \widetilde{\bm{U}}^{\mathrm{T}}\widetilde{\bm{X}},
\end{equation}
and thus Eq. (\ref{Eq3}) becomes
\begin{equation}
 \small
\begin{aligned}
&\frac{1}{N}\sum_{k=1}^K \sum_{i \in {\mathcal{C}_k}} \lVert \mathbf{y}_i - \mathbf{c}_k \rVert^2 \\
=&\frac{1}{N}\mathrm{Tr}(\widetilde {\bm{X}}^{\mathrm{T}}\widetilde {\bm{U}}\widetilde {\bm{U}}^{\mathrm{T}}\widetilde {\bm{X}})
-\frac{1}{N}\mathrm{Tr}({\bm{H}}^{\mathrm{T}}\widetilde {\bm{X}}^{\mathrm{T}}\widetilde {\bm{U}}\widetilde {\bm{U}}^{\mathrm{T}}\widetilde {\bm{X}}\bm{H}).
\end{aligned}
\end{equation}

As for the second term, we can also rewrite it as follow:
\begin{equation}
 \small
\lambda \sum_{p\neq q} \lVert \bm{U}^p-\bm{U}^q\rVert_F^2 = \lambda\mathrm{Tr}(\widetilde{\bm{U}}^{\mathrm{T}}\bm{D\widetilde U}),
\end{equation}
where
\begin{equation}
 \small
\bm{D} =
\begin{bmatrix}
(V-1)\bm{I}& -\bm{I}& -\bm{I}&\cdots &-\bm{I} \\
-\bm{I}& (V-1)\bm{I}& -\bm{I}&\cdots &-\bm{I} \\
\vdots&\vdots&\vdots&\vdots&\vdots \\
-\bm{I}&  -\bm{I}& -\bm{I}& \cdots&(V-1)\bm{I}
\end{bmatrix}.
\end{equation}
Then, it is reasonable to relax the constraints
\begin{equation}
 \small
\bm{U}^{p\mathrm{T}}\bm{\Sigma}^p\bm{U}^p = \bm{I} \quad (p = 1,\cdots,V)
\end{equation}
to
\begin{equation}
 \small
\sum_{p=1}^V \bm{U}^{p\mathrm{T}}\bm{\Sigma}^p\bm{U}^p = \widetilde {\bm{U}}^{\mathrm{T}}\widetilde{\bm{\Sigma}}\widetilde {\bm{U}} = V\bm{I},
\end{equation}
where $\widetilde{\bm{\Sigma}} = diag(\bm{\Sigma}^1, \cdots, \bm{\Sigma}^V)$
because what we expect is to prevent each $\bm{U}^p$ from shrinking to a zero matrix.
The relaxed version of constraints is able to satisfy our need, and it
bypasses trivial computations.

By now we can rewrite our optimization task as follow:
\begin{equation}\label{optFinal}
 \small
\begin{aligned}
\mathop{\min}_{\widetilde{\bm{U}}}\mathcal{F}_{obj} &=
\frac{1}{N}\mathrm{Tr}(\widetilde {\bm{X}}^{\mathrm{T}}\widetilde {\bm{U}}\widetilde {\bm{U}}^{\mathrm{T}}\widetilde {\bm{X}})
+\lambda\mathrm{Tr}(\widetilde{\bm{U}}^{\mathrm{T}}\bm{D\widetilde U}) \\
 &- \frac{1}{N}\mathrm{Tr}({\bm{H}}^{\mathrm{T}}\widetilde {\bm{X}}^{\mathrm{T}}\widetilde {\bm{U}}\widetilde {\bm{U}}^{\mathrm{T}}\widetilde {\bm{X}}\bm{H})
 \\
&s.t.\qquad \widetilde {\bm{U}}^{\mathrm{T}}\widetilde{\bm{\Sigma}}\widetilde {\bm{U}} = V\bm{I}.
\end{aligned}
\end{equation}

It is easy to realize from Eq. (\ref{Eq2}) that our objective function
is highly non-linear and non-convex. Fortunately, in the form of Eq. (\ref{optFinal})
we can find that once $\bm{H}$ is fixed, Lagrange's method can be applied to
our optimization task. And again from Eq. (\ref{Eq2}),
it is exactly the objective of $k$-means clustering once $\widetilde{\bm{U}}$ is fixed \cite{KMEANS}.
Thus, we can adopt an alternating algorithm to solve the optimization problem.

\noindent \textbf{Fix $\bm{H}$ and optimize $\widetilde{\bm{U}}$.} Now we see how we optimize $\widetilde{\bm{U}}$.
After fixing $\bm{H}$ and applying the method
of Lagrange multiplier, our optimization task (\ref{optFinal})
is transformed into an eigen-decomposition problem as follow:
\begin{equation}\label{EqEigenDe}
 \small
\bm{G}\mathbf{u} = \gamma \mathbf{u},
\end{equation}
where $\gamma$ is the Lagrange multiplier (and also is the eigenvalue here) and
\begin{equation}\label{EqG}
 \small
\bm{G} = \widetilde{\bm{\Sigma}}^{-1}(\lambda \bm{D}+\frac{1}{N}\widetilde{\bm{X}}\widetilde{\bm{X}}^{\mathrm{T}}-\frac{1}{N}\widetilde{\bm{X}}\bm{HH}^{\mathrm{T}}\widetilde{\bm{X}}^{\mathrm{T}}).
\end{equation}
Then, $\widetilde{\bm{U}}$ can be obtained by solving this eigen-decomposition problem.

\noindent \textbf{Fix $\widetilde{\bm{U}}$ and optimize $\bm{H}$.} As for the optimization of $\bm{H}$, we can simply fix $\widetilde{\bm{U}}$
and conduct $k$-means clustering in the learned space. Each column of $\bm{H}$,
$\mathbf{h}_k$, is thus constructed according to the clustering result.

Based on the analysis above, we can now propose the main algorithm
of CAMEL in Algorithm \ref{AlgCamel}. We set maximum iteration to 100.
After obtaining $\widetilde{\bm{U}}$, we decompose it back into $\{\bm{U}^1,\cdots,\bm{U}^V\}$.
The algorithm is guaranteed to convergence, as given in the following proposition:
\final{
\begin{prop}
In Algorithm \ref{AlgCamel}, $\mathcal{F}_{obj}$ is guaranteed to convergence.
\end{prop}
\begin{proof}
In each iteration, when $\widetilde{\bm{U}}$ is fixed,
if $\bm{H}$ is the local minimizer, $k$-means remains $\bm{H}$ unchanged, otherwise it seeks the local minimizer.
When $\bm{H}$ is fixed, $\widetilde{\bm{U}}$ has a closed-form solution which is the global minimizer.
Therefore, the $\mathcal{F}_{obj}$ decreases step by step.
As $\mathcal{F}_{obj}\geq 0$ has a lower bound $0$, it is guaranteed to convergence.
\end{proof}
}

%
\begin{algorithm}[t]\label{AlgCamel}
\scriptsize
\caption{CAMEL}
\SetKwInOut{Input}{Input}
\SetKwInOut{Output}{Output}
\Input{$\widetilde{\bm{X}},K,\epsilon=10^{-8}$}
\Output{$\widetilde{\bm{U}}$}
Conduct $k$-means clustering with respect to
each column of $\widetilde{\bm{X}}$ to initialize $\bm{H}$ according to Eq. (\ref{EqH}) and (\ref{EqColH}). \\
Fix $\bm{H}$ and solve the eigen-decomposition problem described by Eq. (\ref{EqEigenDe}) and (\ref{EqG})
to construct $\widetilde{\bm{U}}$. \\
\While{decrement of $\mathcal{F}_{obj} > \epsilon$ \& maximum iteration unreached}
{
\begin{itemize}[leftmargin=*]
\setlength{\topsep}{1ex}
\setlength{\itemsep}{-0.1ex}
\setlength{\parskip}{0.1\baselineskip}
\vspace{0.1cm}
\item Construct $\bm{Y}$ according to Eq. (\ref{EqY}). \\
\item Fix $\widetilde{\bm{U}}$ and conduct $k$-means clustering with respect to
each column \par of $\bm{Y}$ to update $\bm{H}$ according to Eq. (\ref{EqH}) and (\ref{EqColH}). \\
\item Fix $\bm{H}$ and solve the eigen-decomposition problem described by \par Eq. (\ref{EqEigenDe}) and (\ref{EqG})
to update $\widetilde{\bm{U}}$.
\end{itemize}
}
\end{algorithm}

\section{Experiments}

\subsection{Datasets}

\begin{figure}
\begin{center}
\subfigure[]{
\includegraphics[width=0.137\linewidth]{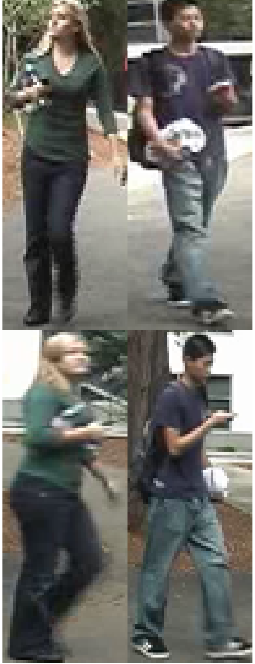}
}
\subfigure[\label{FigDatasetsCUHK01}]{
\includegraphics[width=0.137\linewidth]{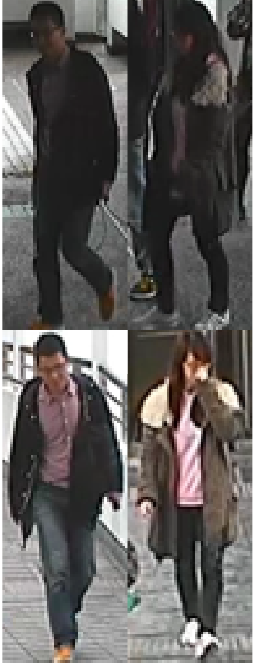}
}
\subfigure[]{
\includegraphics[width=0.137\linewidth]{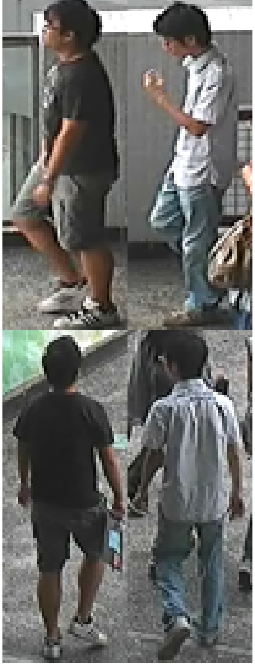}
}
\subfigure[\label{FigDatasetsSYSU}]{
\includegraphics[width=0.137\linewidth]{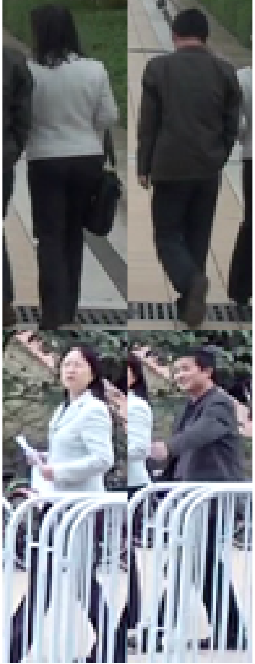}
}
\subfigure[]{
\includegraphics[width=0.137\linewidth]{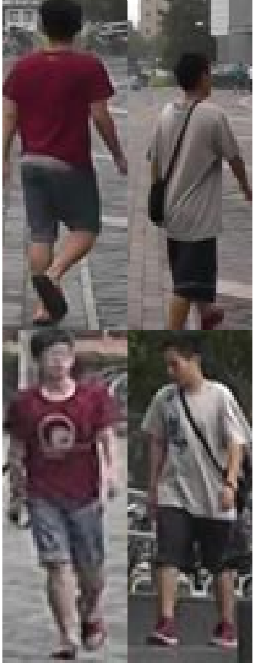}
}
\subfigure[]{
\includegraphics[width=0.137\linewidth]{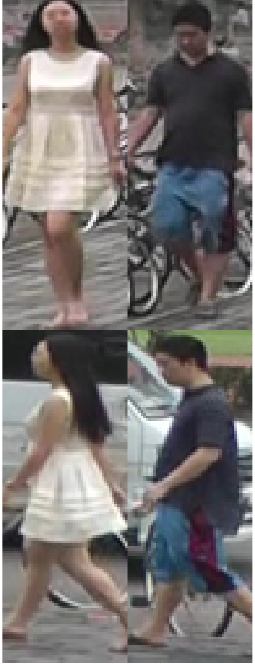}
}
\caption{\label{FigDatasets}Samples of the datasets. Every two images in
a column are from one identity across two disjoint camera views.
(a) VIPeR (b) CUHK01 (c) CUHK03 (d) SYSU (e) Market (f) ExMarket. (Best viewed in color)}
\end{center}
\end{figure}

\begin{table}[t]
\begin{center}
\scriptsize
\begin{tabular}{
>{\centering\arraybackslash}p{1.2cm}
>{\centering\arraybackslash}p{0.5cm}
>{\centering\arraybackslash}p{0.7cm}
>{\centering\arraybackslash}p{0.8cm}
>{\centering\arraybackslash}p{0.7cm}
>{\centering\arraybackslash}p{0.7cm}
>{\centering\arraybackslash}p{0.8cm}}
\toprule
Dataset      & VIPeR & CUHK01 & CUHK03 & SYSU  & Market & ExMarket \\
\midrule
\# Samples & 1,264 & 3,884 & 13,164 & 24,448 & 32,668 & 236,696 \\
\# Views & 2     & 2     & 6     & 2     & 6     & 6 \\
\bottomrule
\end{tabular}%

\caption{\label{TableDatasets}Overview of dataset scales. ``\#'' means ``the number of''.}
\end{center}
\end{table}

Since unsupervised models are more meaningful when the scale of problem
is larger, our experiments were conducted on relatively big datasets
except VIPeR \cite{VIPER} which is small but widely used.
Various degrees of view-specific bias can be observed in all these datasets (see Figure \ref{FigDatasets}).

\noindent \textbf{The VIPeR dataset}  contains 632 identities,
with two images captured from two camera views of each identity.

\noindent \textbf{The CUHK01 dataset} \cite{CUHK01} contains 3,884 images of
971 identities captured from
two disjoint views. There are two images of every identity from each view.

\noindent \textbf{The CUHK03 dataset} \cite{2014_CVPR_CUHK03} contains 13,164 images
of 1,360 pedestrians captured from six surveillance camera views.
Besides hand-cropped images, samples detected
by a state-of-the-art pedestrian detector are provided.

\noindent \textbf{The SYSU dataset} \cite{2015_TCSVT_ASM} includes 24,448 RGB images of 502 persons under two surveillance cameras.
One camera view mainly
captured the frontal or back views of persons, while the other observed mostly
the side views.

\noindent \textbf{The Market-1501 dataset} \cite{2015_ICCV_MARKET} (Market) contains 32,668 images of 1,501 pedestrians, each of which was
captured by at most six cameras. All of the images were cropped by a pedestrian
detector. There are some bad-detected samples in this datasets as distractors
as well.

\noindent \textbf{The ExMarket dataset}\final{\footnote{Demo code for the model and the ExMarket dataset can be found on \url{https://github.com/KovenYu/CAMEL}.}}. In order to evaluate unsupervised RE-ID methods on even larger scale,
which is more realistic, we further combined \textbf{the MARS dataset} \cite{MARS} with
Market. MARS is a video-based RE-ID dataset which contains
20,715 tracklets of 1,261 pedestrians. All the identities from MARS are of a
subset of those from Market.
We then took 20\% frames (each one in every five successive frames) from the tracklets
and combined them with Market to obtain an extended version of Market (\textbf{ExMarket}).
The imbalance between the numbers of samples from the 1,261 persons and other
240 persons makes this dataset more challenging and realistic. There are 236,696 images
in ExMarket in total, and 112,351 images of them are of training set.
A brief overview of the dataset scales can be found in Table \ref{TableDatasets}.

\subsection{Settings}

\noindent \textbf{Experimental protocols}:
A widely adopted protocol was followed on VIPeR in our
experiments \cite{2015_CVPR_LOMO}, i.e., randomly dividing the 632 pairs of images into
two halves, one of which was used as training set and the other as testing set. This
procedure was repeated 10 times to offer average performance.
Only
single-shot experiments were conducted.

The experimental protocol for CUHK01 was the same as that in \cite{2015_CVPR_LOMO}.
We randomly selected 485 persons as training set and the other 486 ones as testing set.
The evaluating procedure was repeated 10 times. Both multi-shot and single-shot
settings were conducted.

The CUHK03 dataset was provided together with its recommended evaluating protocol \cite{2014_CVPR_CUHK03}.
We followed the provided protocol, where images of 1,160 persons were chosen as training set,
images of another 100 persons as
validation set and the remainders as testing set.
This procedure was repeated 20 times.
In our experiments, detected samples were adopted since they
are closer to real-world settings.
Both multi-shot and single-shot experiments were conducted.

As for the SYSU dataset, we randomly picked 251 pedestrians' images as training set
and the others as testing set.
In the testing stage, we basically followed the protocol as in \cite{2015_TCSVT_ASM}. That is,
we randomly chose one and three images of each pedestrian as gallery for single-shot and multi-shot experiments, respectively.
We repeated the testing procedure by 10 times.

Market is somewhat different from others. The evaluation protocol was also
provided along with the data \cite{2015_ICCV_MARKET}. Since the images of one person
came from at most six views, single-shot experiments were not suitable. Instead,
multi-shot experiments were conducted and both cumulative matching characteristic (CMC) and
mean average precision (MAP) were adopted for evaluation \cite{2015_ICCV_MARKET}.
The protocol of ExMarket was identical to that of Market since the identities were
completely the same as we mentioned above.

\noindent \textbf{Data representation}:
In our experiments we used the deep-learning-based JSTL feature proposed in \cite{2016_CVPR_JSTL}.
We implemented it using the 56-layer ResNet \cite{2016_CVPR_resnet}, which
produced $64$-D features.
The original JSTL was adopted to our implementation to extract features on SYSU, Market and ExMarket.
Note that the training set of the original JSTL contained VIPeR, CUHK01 and CUHK03,
violating the unsupervised setting.
So we trained a new JSTL model without VIPeR in its training set to extract
features on VIPeR. The similar procedures were done for CUHK01 and CUHK03.


\noindent \textbf{Parameters}:
We set $\lambda$, the cross-view consistency regularizer, to $0.01$.
We also evaluated the situation when $\lambda$ goes to infinite, i.e.,
the symmetric version of our model in Sec. \ref{SecFurtherEval},
to show how important the asymmetric modelling is.

\Koven{Regarding the parameter $T$ which is the feature dimension after the transformation learned by CAMEL, we set $T$ equal to original feature dimension i.e., $64$, for simplicity. In our experiments, we found that CAMEL can align data distributions across camera views even without performing any further dimension reduction.
This may be due to the fact that, unlike conventional subspace learning models, the transformations learned by CAMEL are view-specific for different camera views and always non-orthogonal. Hence, the learned view-specific transformations can already reduce the discrepancy between the data distributions of different camera views.}

As for $K$, we found that
our model was not sensitive to $K$ when $N\gg K$ and $K$ was not too small
(see Sec. \ref{SecFurtherEval}),
so we set $K = 500$.
These parameters were fixed for all datasets.

\subsection{Comparison}\label{SecFairCmp}

Unsupervised models are more significant when applied on larger datasets.
In order to make comprehensive and fair comparisons, in this section
we compare CAMEL with the most comparable unsupervised models
on six datasets with their scale orders varying from hundreds to hundreds of thousands.
We show the comparative results measured by
the rank-1 accuracies of CMC and MAP (\%)
in Table \ref{TableJSTL}.

\noindent \textbf{Comparison to Related Unsupervised RE-ID Models}.
In this subsection we compare CAMEL with the sparse dictionary learning
model (denoted as Dic) \cite{2015_BMVC_DIC},
sparse representation learning model ISR \cite{2015_PAMI_ISR},
kernel subspace learning model RKSL \cite{2016_ICIP_Wang} and
sparse auto-encoder (SAE) \cite{SAE1,SAE2}.
We tried several sets of parameters for them, and report the best ones.
We also adopt the Euclidean distance which is adopted in the original JSTL paper \cite{2016_CVPR_JSTL} as a baseline (denoted as JSTL).

From Table \ref{TableJSTL}
we can observe that
CAMEL outperforms other models on all the datasets on both settings.
In addition, we can further see from Figure \ref{FigCMC} that CAMEL outperforms other models
at any rank.
One of the main reasons is that the view-specific
interference is noticeable in these datasets. For example, we can see in Figure \ref{FigDatasetsCUHK01} that
on CUHK01, the
changes of illumination are extremely severe and even human beings may have difficulties in
recognizing the identities in those images across views.
This impedes other symmetric models from achieving higher accuracies,
because they potentially hold an assumption that
the invariant and discriminative information can be retained and exploited through a universal
transformation for all views.
But CAMEL relaxes this assumption by
learning an asymmetric metric and then can outperform other models significantly.
In Sec. \ref{SecFurtherEval} we will see the performance of CAMEL would drop much
when it degrades to a symmetric model.


\begin{table}[t]
\scriptsize
\begin{center}
\setlength{\tabcolsep}{0.16cm}
\begin{tabular}{
>{\centering\arraybackslash}p{1.2cm}
>{\centering\arraybackslash}p{0.7cm}
>{\centering\arraybackslash}p{0.8cm}
>{\centering\arraybackslash}p{0.85cm}
>{\centering\arraybackslash}p{0.85cm}
>{\centering\arraybackslash}p{0.85cm}
>{\centering\arraybackslash}p{0.85cm}}
\toprule
Dataset      & VIPeR & CUHK01 & CUHK03 & SYSU  & Market & ExMarket \\
      \midrule
Setting      & SS    & SS/MS  & SS/MS  & SS/MS & MS & MS \\
      \midrule
Dic \begin{tiny}\cite{2015_BMVC_DIC}\end{tiny}   &29.9&49.3/52.9&27.4/36.5&21.3/28.6&50.2(22.7)& 52.2(21.2) \\
ISR \begin{tiny}\cite{2015_PAMI_ISR}\end{tiny}  &27.5 &53.2/55.7 &31.1/38.5& 23.2/33.8& 40.3(14.3)&- \\
RKSL \begin{tiny}\cite{2016_ICIP_Wang}\end{tiny} &25.8 & 45.4/50.1 &25.8/34.8 &17.6/23.0 &34.0(11.0) &- \\
SAE \begin{tiny}\cite{SAE1}\end{tiny}  &20.7 &45.3/49.9 &21.2/30.5 &18.0/24.2 &42.4(16.2) &44.0(15.1) \\
JSTL \begin{tiny}\cite{2016_CVPR_JSTL}\end{tiny} &25.7 &46.3/50.6 &24.7/33.2 &19.9/25.6 &44.7(18.4) &46.4(16.7)\\
\midrule
AML \begin{tiny}\cite{2007_CVPR_AML}\end{tiny}  &23.1 &46.8/51.1 &22.2/31.4 &20.9/26.4 &44.7(18.4) &46.2(16.2) \\
UsNCA \begin{tiny}\cite{2015_NC_uNCA}\end{tiny} &24.3 &47.0/51.7 &19.8/29.6 &21.1/27.2 &45.2(18.9) &- \\
\midrule
CAMEL & \textbf{30.9} & \textbf{57.3/61.9} & \textbf{31.9/39.4} & \textbf{30.8/36.8} & \textbf{54.5}(\textbf{26.3}) & \textbf{55.9}(\textbf{23.9})  \\
\bottomrule
\end{tabular}%

\caption{\label{TableJSTL}Comparative results of unsupervised models on the six datasets, measured by
rank-1 accuracies and MAP (\%).
``-'' means prohibitive time consumption due to time complexities of the models.
``SS'' represents single-shot setting and ``MS'' represents multi-shot setting.
For Market and ExMarket, MAP is also provided in the parentheses due to the
requirement in the protocol \cite{2015_ICCV_MARKET}.
Such a format is also applied in the following tables.}
\end{center}
\end{table}

\begin{figure}
\begin{center}
\subfigure[VIPeR]{
\includegraphics[width=0.47\linewidth]{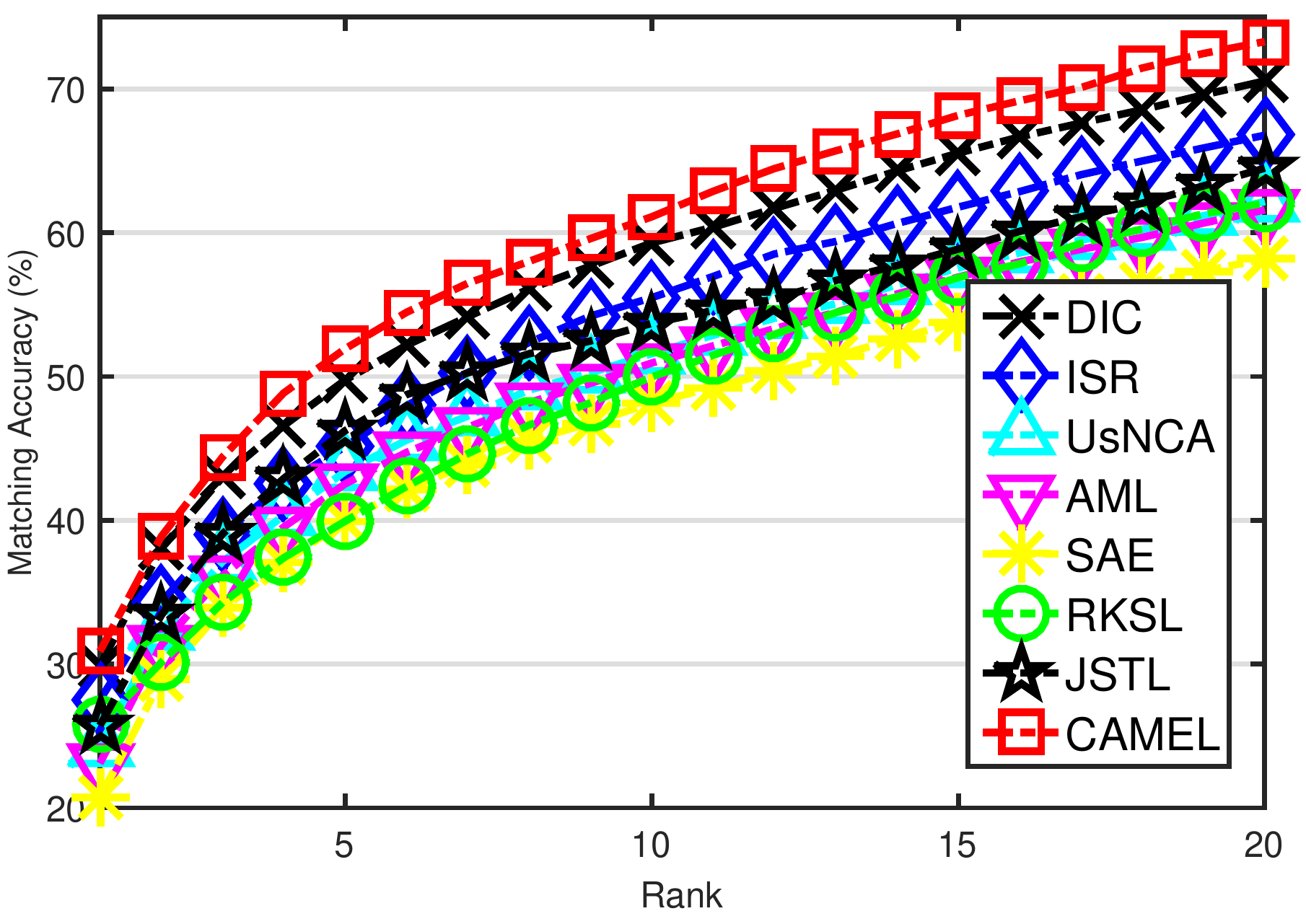}
}
\subfigure[CUHK01]{
\includegraphics[width=0.47\linewidth]{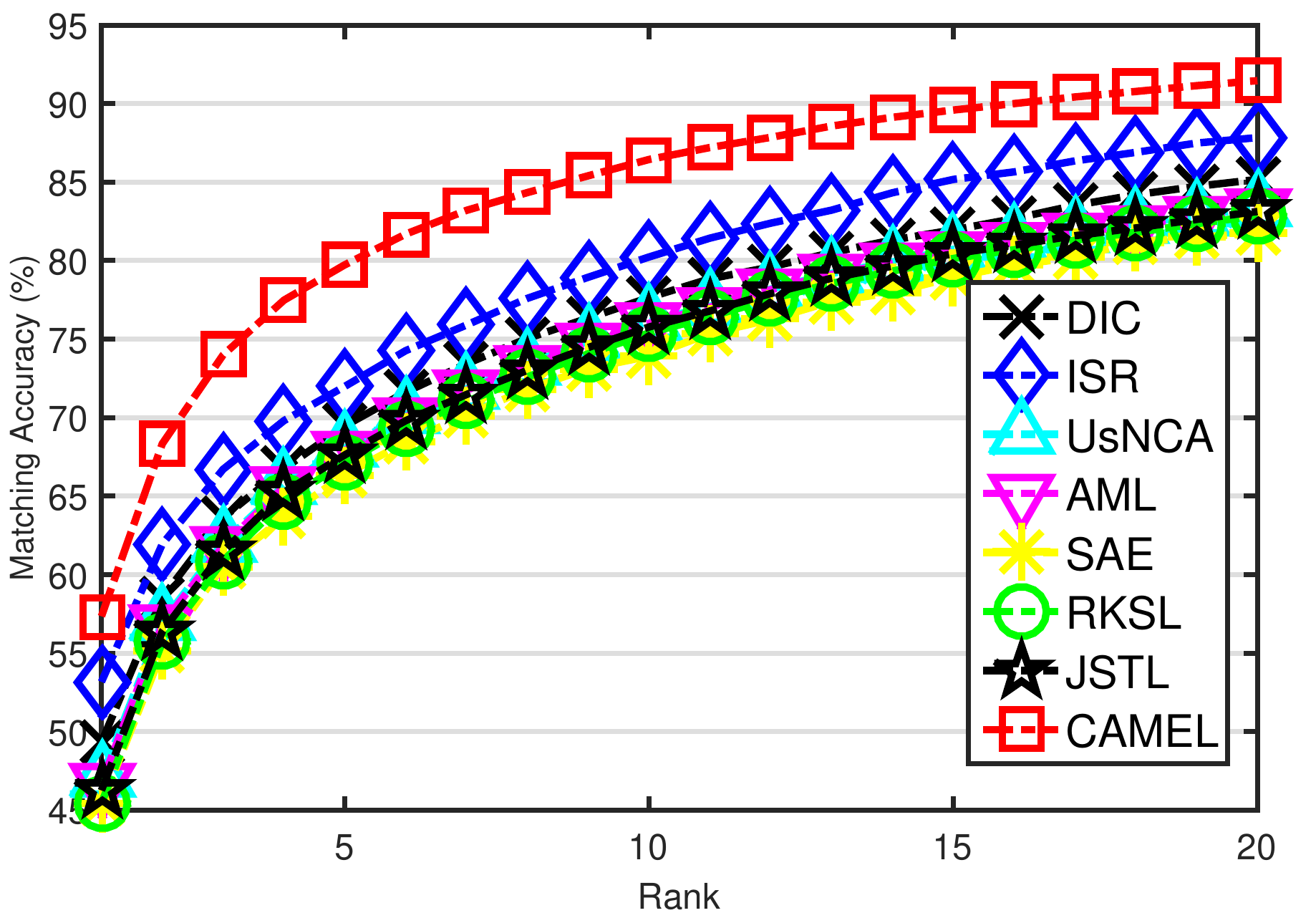}
}
\subfigure[CUHK03]{
\includegraphics[width=0.47\linewidth]{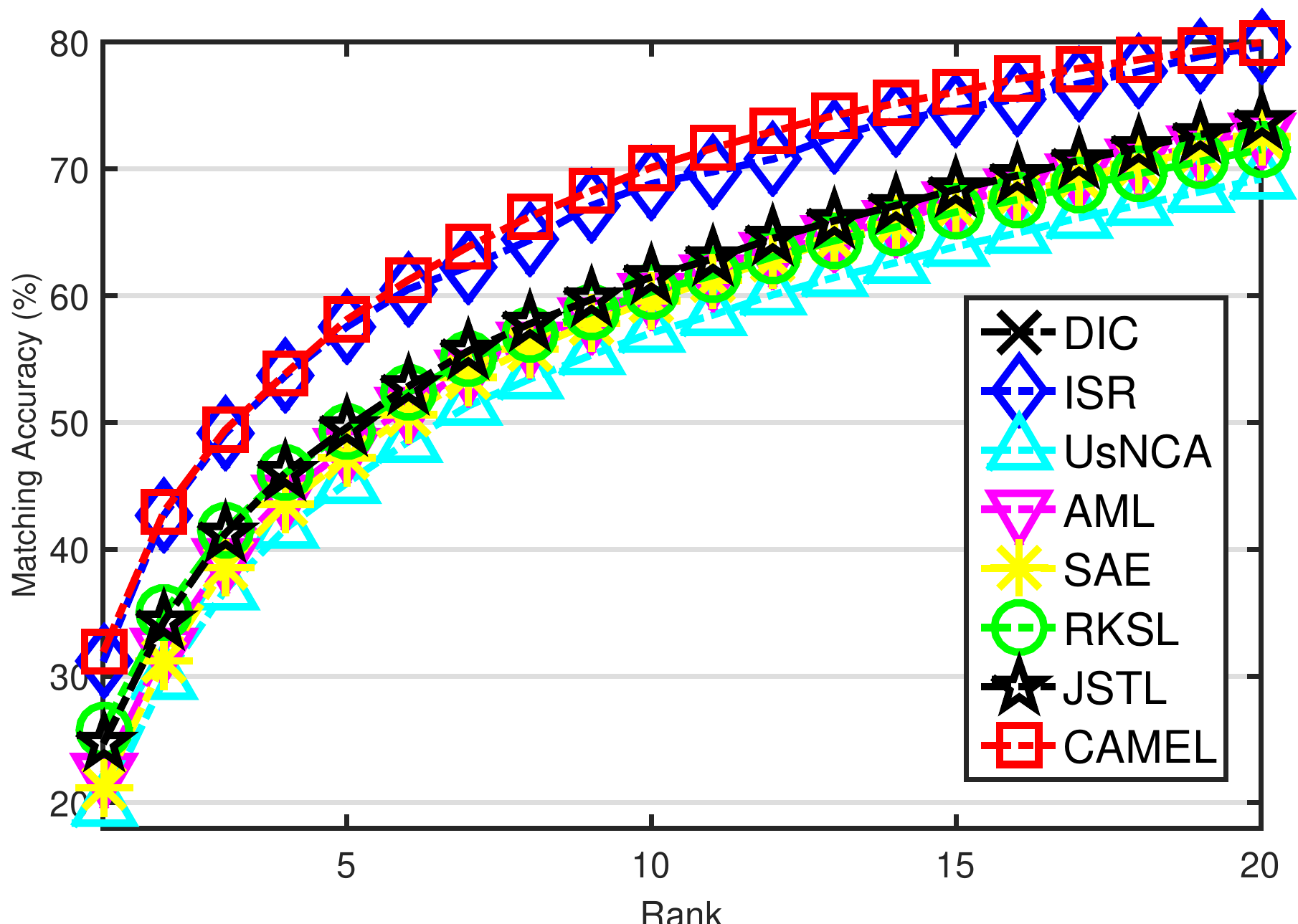}
}
\subfigure[SYSU]{
\includegraphics[width=0.47\linewidth]{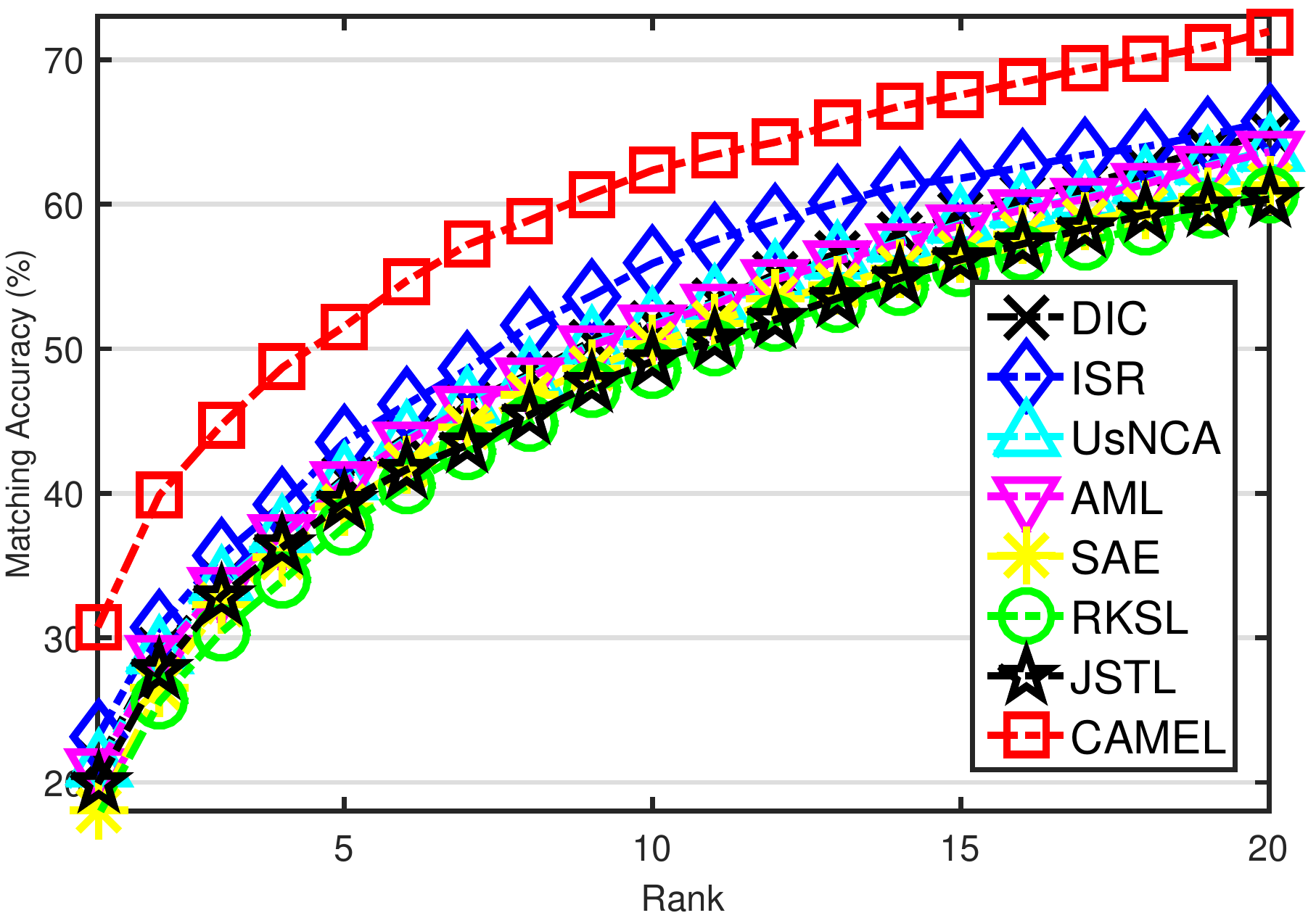}
}
\subfigure[Market]{
\includegraphics[width=0.47\linewidth]{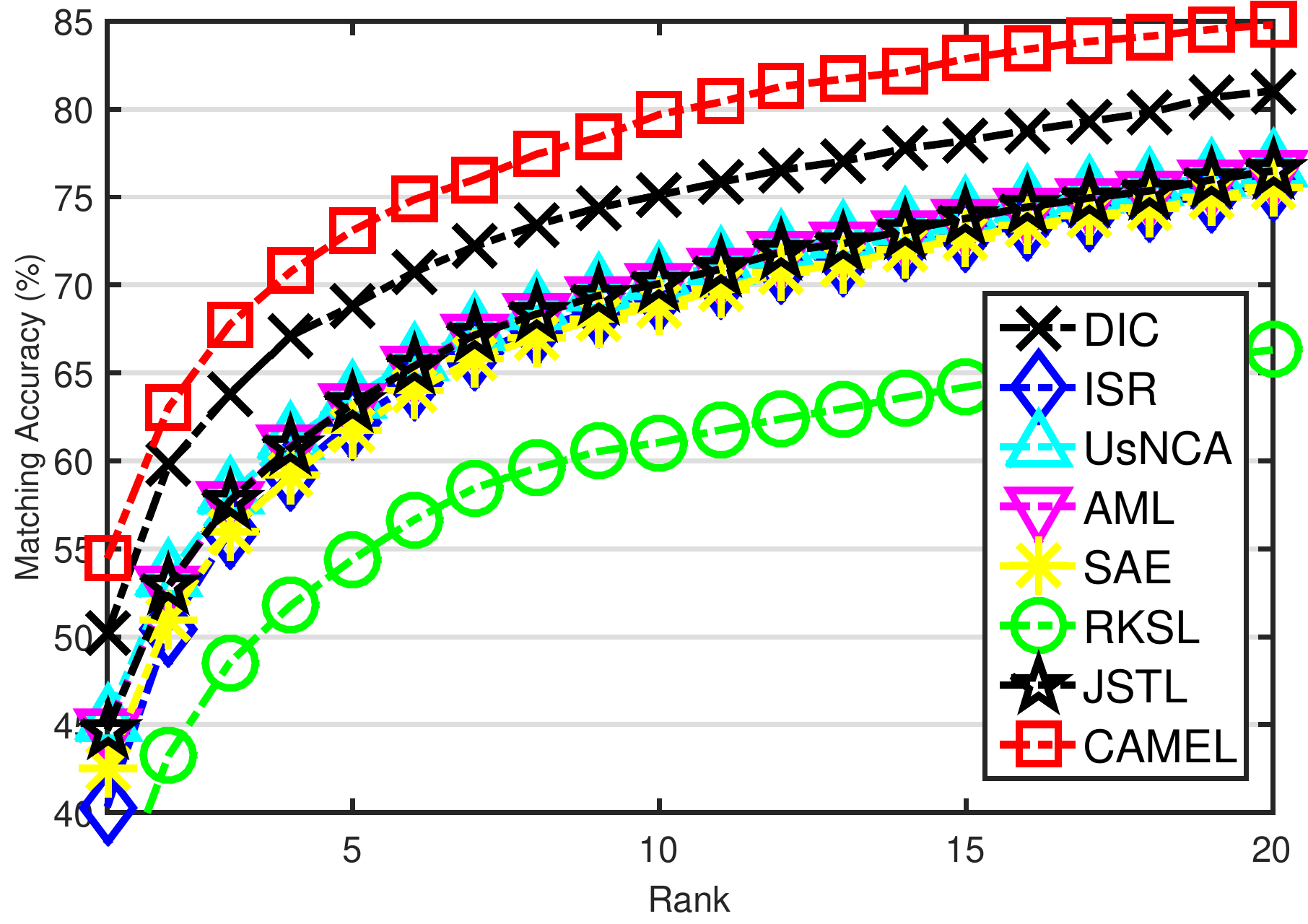}
}
\subfigure[ExMarket]{
\includegraphics[width=0.47\linewidth]{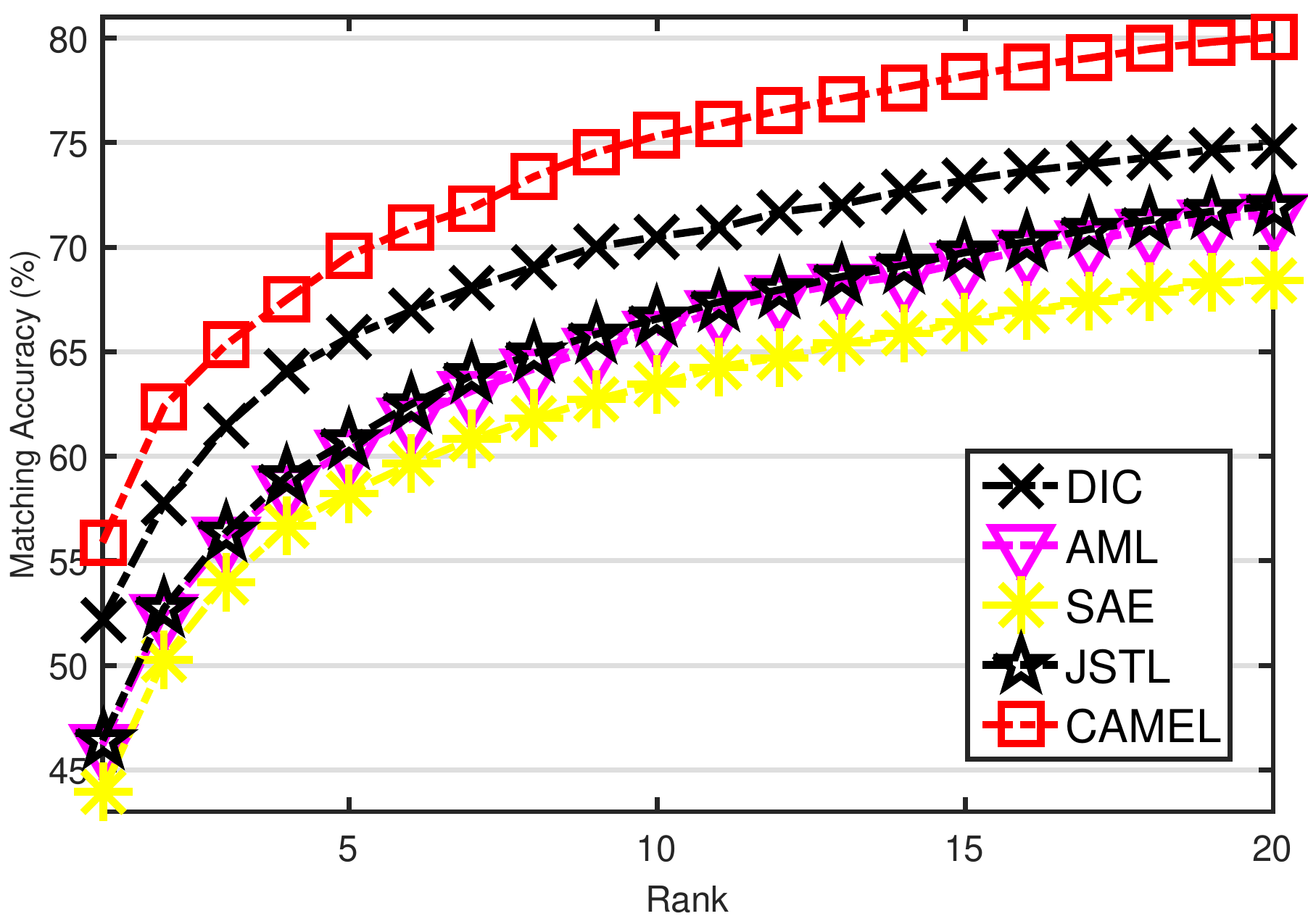}
}
\caption{\label{FigCMC}CMC curves.
\modify{For CUHK01, CUHK03 and SYSU, we take the results under single-shot setting as examples.
Similar patterns can be observed on multi-shot setting.}}

\end{center}
\end{figure}


\noindent \textbf{Comparison to Clustering-based Metric Learning Models}.
In this subsection we compare CAMEL with
a typical model AML \cite{2007_CVPR_AML} and a recently proposed model UsNCA \cite{2015_NC_uNCA}.
We can see from Fig. \ref{FigCMC} and Table \ref{TableJSTL} that compared to them, CAMEL achieves
noticeable improvements on all the six datasets.
One of the major reasons is that
they do not consider the view-specific bias which can
be very disturbing in clustering, making them unsuitable for RE-ID problem.
\ws{In comparison}, CAMEL alleviates such disturbances by asymmetrically modelling.
This factor contributes to the much better performance of CAMEL.

\noindent \textbf{Comparison to the State-of-the-Art.}
In the last subsections, we compared with existing unsupervised RE-ID methods using the same features.
In this part, we also compare with the results reported in literatures.
Note that most existing unsupervised RE-ID methods have not been evaluated on large datasets like CUHK03, SYSU, or Market,
so Table \ref{TableSotA} only reports the comparative results
on VIPeR and CUHK01.
We additionally compared existing unsupervised RE-ID models, including the
hand-craft-feature-based SDALF \cite{2010_CVPR_SDALF} and CPS \cite{CAVIAR},
the transfer-learning-based UDML \cite{2016_CVPR_tDIC},
graph-learning-based model (denoted as GL) \cite{2016_ECCV_Kodirov},
and local-salience-learning-based GTS \cite{2014_BMVC_GTS} and SDC \cite{2013_CVPR_SALIENCE}.
We can observe from Table \ref{TableSotA} that
our model CAMEL can outperform the state-of-the-art by large margins on CUHK01.

\begin{table}[t]
\begin{center}
\scriptsize
\begin{tabular}{cccccccc}
\toprule
Model     & SDALF   & CPS & UDML
& GL  & GTS
& SDC  & CAMEL \\
&\cite{2010_CVPR_SDALF} &\cite{CAVIAR} &\cite{2016_CVPR_tDIC} &\cite{2016_ECCV_Kodirov}  &\cite{2014_BMVC_GTS} & \cite{2013_CVPR_SALIENCE} & \\
\midrule
VIPeR & 19.9  & 22.0  & 31.5  & \textbf{33.5}  & 25.2    & 25.8 & 30.9 \\
CUHK01 & 9.9  & -  & 27.1  & 41.0 & -          & 26.6 & \textbf{57.3} \\
\bottomrule
\end{tabular}%

\caption{\label{TableSotA}Results compared to the state-of-the-art reported in literatures, measured by rank-1 accuracies (\%). ``-'' means no reported result.}

\end{center}
\end{table}

\noindent \textbf{Comparison to Supervised Models.}
Finally, in order to see how well CAMEL can approximate the performance of supervised RE-ID,
\Koven{we additionally compare CAMEL with its supervised version (denoted as CAMEL$_s$) which is easily derived by substituting the clustering results by true labels, and three standard supervised models,
including the widely used KISSME \cite{2012_CVPR_KISSME}, XQDA \cite{2015_CVPR_LOMO}, the asymmetric distance model CVDCA \cite{2015_TCSVT_ASM}.
The results are shown in Table \ref{TableSupervised}.
We can see that CAMEL$_s$ outperforms CAMEL by various degrees,
indicating that label information can further improve CAMEL's performance.
Also from Table \ref{TableSupervised}, we notice that CAMEL can be comparable to other standard supervised models on some datasets like CUHK01,
and even outperform some of them.}
It is probably because the used JSTL model had not been fine-tuned on the target datasets: this was for a fair comparison with unsupervised models which work on completely unlabelled training data.
Nevertheless, this suggests that the performance of CAMEL may not be far below the standard supervised RE-ID models.

\begin{table}[t]
\scriptsize
\setlength{\tabcolsep}{0.11cm}
\begin{tabular}{ccccccc}
\toprule
Dataset      & VIPeR & CUHK01 & CUHK03 & SYSU  & Market & ExMarket \\
\midrule
Setting      & SS    & SS/MS  & SS/MS  & SS/MS & MS & MS \\
      \midrule
KISSME \begin{tiny}\cite{2012_CVPR_KISSME}\end{tiny} &28.4&53.0/57.1&37.8/45.4&24.7/31.8&51.1(24.5)& 48.0(18.3) \\
XQDA \begin{tiny}\cite{2015_CVPR_LOMO}\end{tiny}  &28.9&54.3/58.2&36.7/43.7&25.2/31.7&50.8(24.4)& 47.4(18.1) \\
CVDCA \begin{tiny}\cite{2015_TCSVT_ASM}\end{tiny} &\textbf{37.6}&57.1/60.9&37.0/44.6&31.1/\textbf{38.9}&52.6(25.3)&51.5(22.6) \\
CAMEL$_s$ &33.7&\textbf{58.5/62.7}&\textbf{45.1/53.5}&\textbf{31.6}/37.6&\textbf{55.0}(\textbf{27.1})& \textbf{56.1}(\textbf{24.1}) \\
\midrule
CAMEL & 30.9 & 57.3/61.9 & 31.9/39.4 & 30.8/36.8 &54.5(26.3) & 55.9(23.9) \\
\bottomrule
\end{tabular}%

\caption{\label{TableSupervised}Results compared to supervised models using the same JSTL features.}
\end{table}

\subsection{Further Evaluations}\label{SecFurtherEval}

\noindent \textbf{The Role of Asymmetric Modeling}.
We show what is going to happen if CAMEL degrades to a common symmetric model
in Table \ref{TableSym}. Apparently, without asymmetrically modelling each camera view,
our model would be worsen largely, indicating that the asymmetric modeling for clustering
is rather important for addressing the cross-view matching problem in RE-ID as well as in our model.

\noindent \textbf{Sensitivity to the Number of Clustering Centroids}. We take
CUHK01, Market and ExMarket datasets as examples of different scales (see Table \ref{TableDatasets}) for this evaluation.
Table \ref{TableK} shows how the performance varies with different numbers of clustering centroids, $K$.
It is obvious that the performance
only fluctuates mildly when $N \gg K$ and $K$ is not too small.
Therefore CAMEL is not very sensitive to $K$ especially when applied to large-scale problems.
\final{To further explore the reason behind,
we show in Table \ref{table:rate} the rate of clusters which contain more than one persons,
in the initial stage and convergence stage in Algorithm \ref{AlgCamel}.
We can see that \emph{(1)} in spite of that $K$ is varying,
there is always a number of clusters containing more than one persons in both the initial stage and convergence stage.
This indicates that our model works \emph{without} the requirement of perfect clustering results.
And \emph{(2)}, although the number is various,
in the convergence stage the number is consistently decreased compared to initialization stage.
This shows that the cluster results are improved consistently.
These two observations suggests that
the clustering should be a mean to learn the asymmetric metric, rather than an ultimate objective.}

\modify{
\noindent \textbf{Adaptation Ability to Different Features}.
At last, we show that CAMEL can be effective not only when adopting deep-learning-based JSTL features.
We additionally adopted the hand-crafted LOMO feature proposed in \cite{2015_CVPR_LOMO}.
We performed PCA to produce $512$-D LOMO features, and the results are shown in Table \ref{TableLOMO}.
Among all the models, the results of Dic and ISR are the most comparable (Dic and ISR take all second places). So for clarity, we only compare CAMEL with them and $L_2$ distance as baseline.
From the table we can see that CAMEL can outperform them.
}

\begin{table}[t]
\begin{center}
\scriptsize
\setlength{\tabcolsep}{0.11cm}
\begin{tabular}{ccccccc}
\toprule
Dataset      & VIPeR & CUHK01 & CUHK03 & SYSU  & Market & ExMarket \\
\midrule
Setting      & SS    & SS/MS  & SS/MS  & SS/MS & MS & MS \\
      \midrule
CMEL & 27.5  & 52.5/54.9 & 29.8/37.5 & 25.4/30.9 & 47.6(21.5) & 48.7(20.0) \\
CAMEL & \textbf{30.9} & \textbf{57.3/61.9} & \textbf{31.9/39.4} & \textbf{30.8/36.8} & \textbf{54.5}(\textbf{26.3}) & \textbf{55.9}(\textbf{23.9}) \\
\bottomrule
\end{tabular}%

\caption{\label{TableSym}Performances of CAMEL compared to its symmetric version, denoted as CMEL.}
\end{center}
\end{table}

\begin{table}[t]
\begin{center}
\scriptsize
\begin{tabular}{cccccc}
\toprule
K     & 250   & 500   & 750   & 1000  & 1250 \\
\midrule
CUHK01 & 56.59 & 57.35 & 56.26 & 55.12 & 52.75 \\
Market & 54.48 & 54.45 & 54.54 & 54.48 & 54.48 \\
ExMarket & 55.49 & 55.87 & 56.17 & 55.93 & 55.67 \\
\bottomrule
\end{tabular}%

\caption{\label{TableK}Performances of CAMEL when the number of clusters, K, varies.
Measured by single-shot rank-1 accuracies (\%) for CUHK01 and multi-shot for Market and ExMarket.}
\end{center}
\end{table}

\begin{table}[t]
\begin{center}
\scriptsize
\begin{tabular}{cccccc}
\toprule
K     & 250   & 500   & 750   & 1000  & 1250 \\
\midrule
Initial Stage & 77.6\% & 57.0\% & 26.3\% & 11.6\% & 6.0\% \\
Convergence Stage & 55.8\% & 34.3\% & 18.2\% & 7.2\%  & 4.8\% \\
\bottomrule
\end{tabular}%
\caption{\label{table:rate}
Rate of clusters containing similar persons on CUHK01.
Similar trend can be observed on other datasets.}
\end{center}
\end{table}

\begin{table}[t]
\begin{center}
\scriptsize
\setlength{\tabcolsep}{0.16cm}
\begin{tabular}{
>{\centering\arraybackslash}p{1.2cm}
>{\centering\arraybackslash}p{0.7cm}
>{\centering\arraybackslash}p{0.8cm}
>{\centering\arraybackslash}p{0.85cm}
>{\centering\arraybackslash}p{0.85cm}
>{\centering\arraybackslash}p{0.85cm}
>{\centering\arraybackslash}p{0.85cm}}
\toprule
Dataset      & VIPeR & CUHK01 & CUHK03 & SYSU  & Market & ExMarket \\
      \midrule
Setting      & SS    & SS/MS  & SS/MS  & SS/MS & MS & MS \\
\midrule
Dic \begin{tiny}\cite{2015_BMVC_DIC}\end{tiny}  & 15.8  & 19.6/23.6 & 8.6/13.4 & 14.2/24.4 & 32.8(12.2) & 33.8(12.2) \\
ISR \begin{tiny}\cite{2015_PAMI_ISR}\end{tiny}  & 20.8  & 22.2/27.1 & 16.7/20.7 & 11.7/21.6 & 29.7(11.0) & - \\
$L_2$                                           & 11.6  & 14.0/18.6 & 7.6/11.6 & 10.8/18.9 & 27.4(8.3) & 27.7(8.0) \\
\midrule
CAMEL & \textbf{26.4} & \textbf{30.0/36.2} & \textbf{17.3/23.4} & \textbf{23.6/35.6} & \textbf{41.4(14.1)} & \textbf{42.2(13.7)} \\
\bottomrule
\end{tabular}%
\caption{\label{TableLOMO}Results using $512$-D LOMO features.}
\end{center}
\end{table}

\section{Conclusion}

In this work, we have shown that metric learning can be effective for unsupervised RE-ID by proposing
clustering-based asymmetric metric learning called CAMEL. \ws{CAMEL learns view-specific projections
to deal with view-specific interference, and this is based on existing clustering (e.g., the $k$-means model demonstrated in this work)
on RE-ID unlabelled data, resulting in an asymmetric metric clustering.
}
Extensive experiments show that our model can outperform
existing ones in general, especially on large-scale unlabelled RE-ID datasets.

\section*{Acknowledgement}
This work was supported partially by the National Key Research and Development Program of China (2016YFB1001002), NSFC(61522115, 61472456, 61573387, 61661130157, U1611461), the Royal Society Newton Advanced Fellowship (NA150459), Guangdong Province Science and Technology Innovation Leading Talents (2016TX03X157).
{\small
\bibliographystyle{ieee}
\bibliography{Koven}
}

\end{document}